\documentclass{svjour3}  
\smartqed  
\usepackage{graphicx}
\usepackage{bbold}
\usepackage{optidef}

\usepackage{amsthm,amsmath}
\usepackage[ruled,vlined]{algorithm2e}
\theoremstyle{definition}
\newtheorem{defn}{Definition}
\theoremstyle{plain}
\newtheorem{thm}{Theorem}
\newtheorem{lem}[thm]{Lemma}

\usepackage[utf8]{inputenc}
\usepackage[T1]{fontenc}
\usepackage{optidef}
\usepackage{hyperref}
\hypersetup{
    colorlinks=true,
    linkcolor=blue,
    filecolor=magenta,      
    urlcolor=black,
}
\DeclareUnicodeCharacter{3000}{} 
\usepackage{mathptmx}      
\begin{document}

\title{Learning subtree pattern importance for 　Weisfeiler-Lehman based  graph kernels}


\author{Dai Hai Nguyen         \and Canh Hao Nguyen
        \and Hiroshi Mamitsuka 
}


\institute{Dai Hai Nguyen \at
              Graduate School of Frontier Sciences, The University of Tokyo, 5-1-5 Kashiwa-no-ha, Kashiwa, Chiba 277-8561, Japan \\
              \email{hai@k.u-tokyo.ac.jp}           
           \and
           Canh Hao Nguyen \at
            Bioinformatics Center, Institute for Chemical Research, Kyoto University, Uji 611-0011, Japan\\
            \email{canhhao@kuicr.kyoto-u.ac.jp}  
            \and
            Hiroshi Mamitsuka \at
            Bioinformatics Center, Institute for Chemical Research, Kyoto University, Uji 611-0011, Japan and Department of Computer Science, Alato University, Espoo 02150, Finland\\
            \email{mami@kuicr.kyoto-u.ac.jp}  
}

\date{Received: date / Accepted: date}
\maketitle
\begin{abstract}{Graph is an usual representation of relational data, which are ubiquitous in many domains such as molecules, biological and social networks. A popular approach to learning with graph structured data is to make use of graph kernels, which measure the similarity between graphs and are plugged into a kernel machine such as a support vector machine. Weisfeiler-Lehman (WL) based graph kernels, which employ WL labeling scheme to extract subtree patterns and perform node embedding, are demonstrated to achieve great performance while being efficiently computable. However, one of the main drawbacks of a general kernel is the decoupling of kernel construction and learning process. For molecular graphs, usual kernels such as WL subtree, based on substructures of the molecules, consider all available substructures having the same importance, which might not be suitable in practice. In this paper, we propose a method to learn the weights of subtree patterns in the framework of WWL kernels, the state of the art method for graph classification task \cite{togninalli2019wasserstein}. To overcome the computational issue on large scale data sets, we present an efficient learning algorithm and also derive a generalization gap bound to show its convergence. Finally, through experiments on synthetic and real-world data sets, we demonstrate the effectiveness of our proposed method for learning the weights of subtree patterns.}
\end{abstract}

\section{Introduction}
\label{intro}
Graphs are natural data structures, which appear in various domains such as bioinformatics \cite{sharan2006modeling}, cheminformatics \cite{trinajstic2018chemical}, social network analysis \cite{scott2011social} and so on, where nodes (vertices) represent objects and edges represent the relations between them. A popular approach to learning with graph structured data is to make use of graph kernels. Essentially, a graph kernel is a measure of the similarity between two graphs and must satisfy two fundamental requirements of being a valid kernel: 1) symmetric and 2) positive semi-definite (PSD). Furthermore, the requirements of designing a graph kernel are: it should capture the semantic inherent in the graph structures (e.g. substructures of different levels), and it must be efficiently computable \cite{vishwanathan2010graph}. 

A number of graph kernels have been proposed in literature such as
random walk \cite{kashima2003marginalized}, shortest path \cite{borgwardt2005shortest}, Weisfeiler-Lehman (WL) subtree \cite{shervashidze2009fast} kernels, just to name a few. Most of them are based on $\mathcal{R}$-Convolution framework \cite{haussler1999convolution}, which decomposes two graphs into substructures and adds up the similarities between their substructures to compute kernel values.  Different graph kernels
are defined under different ways of decomposition (types of substructures). For instance,
the substructures can be random walks \cite{kashima2003marginalized}, shortest paths \cite{borgwardt2005shortest} or subtree patterns \cite{vishwanathan2010graph}. Among these, WL subtree kernels have been shown to achieve great prediction performance while being efficiently computable. The key point is that it simply employs a WL based color refinement scheme to embed each node in a given graph into a vector of WL labels, which correspond to \emph{subtree patterns} of the graphs. Then, the kernel between two graphs is defined as the sum of all pairwise similarities between any two node embeddings of the two graphs.

Following a different approach, WL based optimal assignment (WL-OA) kernel \cite{kriege2016valid} assigns one node embedding of one graph to one embedding of the other such that the total similarities
between assigned node embeddings is maximized. This is also known as optimal assignment problem in combinatorial mathematics \cite{munkres1957algorithms}. In a similar vein, Wasserstein WL (WWL, \cite{togninalli2019wasserstein}) uses optimal transport (OT), also known as Wasserstein distance \cite{villani2008optimal}, for measuring the distance between two graphs based on their WL node embeddings. The distance is then converted into a similarity matrix through Laplacian kernel. Furthermore, both of these similarity matrices are shown to be valid kernels due to the hierachy property of WL labels (see \cite{kriege2016valid} and \cite{togninalli2019wasserstein} for more details).

One of the main drawbacks of these kernels is that they are predefined feature extraction without learning the importance of substructures to the problem. This results in the decoupling of data representation and learning process. In these kernels, substructures are given the same weights. However, for the problems such as molecule classification, it is known that only subparts of the molecules are responsible for their properties. Therefore, we wish to be able to give weights to their substructures to have higher classification performance and model interpretation. Based on this motivation, we propose a model to learn the weights of \textit{subtree patterns} (extracted by WL labeling scheme). Our work extends WWL kernels \cite{togninalli2019wasserstein} by formulating an OT based distance as a parametric function of subtree pattern weights before converting into kernels.
We also propose an efficient stochastic learning algorithm to estimate the weights and derive a generalization gap bound for the algorithm. Finally, through experiments on synthetic and four real-world data sets, we show that learning important subtree patterns by our proposed method can lead to more accurate predictive performance and extract important patterns which enhance the classification results.

The remainder of the paper is organized as follows: in Section 2, we review graph kernels which are based on WL labeling scheme, including WL subtree, WL-OA and WWL kernels. In Section 3, we present our method that parameterizes the Wasserstein distance between two graphs with WL labeling scheme as a function of subtree patterns and present the stochastic algorithm for learning parameters of the function. In Section 4, we derive a generalization bound for the learning algorithm. In Section 5, experimental results on the synthetic and real-world data sets are provided. Finally, we conclude by summarizing this work and discussing possible extensions in Section 6.
\section{Related work}
In this paper, we consider the binary classification problem for graph structured data: given a collection of labeled graphs $(g_{i}, y_{i}), i=1,..,n$ (where $n$ is the number of examples) drawn from an unknown joint distribution $\mathcal{P}$ over $\mathcal{G}\times\{-1,1\}$, where $\mathcal{G}$ is a space of graphs. We wish to learn a classifier $h:\mathcal{G}\to \{-1,1\}$, which is based on a similarity function $K:\mathcal{G}\times\mathcal{G}\to [-1,1]$. If $K$ is symmetric and positive semi-definite (PSD), it is called a valid kernel.

There are a number of proposed kernels on graphs, see \cite{kashima2003marginalized,borgwardt2005shortest,shervashidze2009fast}. 
In general, they are defined based on $\mathcal{R}$-Convolution framework \cite{haussler1999convolution}, that is, each graph $g\in\mathcal{G}$ is decomposed into substructures,
and a kernel value $K(g,g^\prime)$ is defined as a sum of pairwise similarities between their substructures. In fact, many graph kernels can be considered as instances of the $\mathcal{R}$-Convolution framework under different decomposition into substructures. The substructures can be random walks \cite{kashima2003marginalized}, shortest paths \cite{borgwardt2005shortest} or circle subtrees \cite{shervashidze2009fast}. Among these, Weisfeiler-Lehman (WL) subtree kernels \cite{shervashidze2009fast} and its variants have been shown to achieve great performance for the graph classification tasks. In this work, we focus on WL based kernels
and will review them in the following subsections.
\subsection{Weisfeiler-Lehman (WL) scheme for node embeddings}
Weisfeiler-Lehman (WL) subtree kernels \cite{shervashidze2009fast} are based on an iterative colour
refinement (also known as WL labeling scheme) and have been shown to achieve
great performance for graph classification task.
For each node of a given graph, the WL labeling scheme creates a sequence of 
ordered strings by the aggregation of the labels of the node and
its neighbors; these strings are then hashed or indexed to produce 
compressed  updated node labels or new indices. If the iteration of the scheme 
is increased, these obtained labels represent increasingly broader 
neighborhood of each node. More specifically, for a graph $G=(V,E)$
with initial labels $\ell_{0}(v)$ for $v\in V$ and let $H$ be the number of
WL iterations, we can define a sequence of refined labels
$(\ell_{0},\ell_{1},...,\ell_{H})$, where $\ell_{h+1}$ is obtained from $\ell_{h}$ by the
following procedure: $\ell_{h+1}(v)=\mathrm{hash}(\ell_{h}(v), \mathcal{N}_{h}(v))$, where $\mathcal{N}_{h}(v)$ denotes a lexicographically sorted sequence of labels of v's neighbors at iteration $h$ and the $\mathrm{hash}$ function is to create a updated compressed node label for $v$. We use perfect hashing for the $\text{hash}$ function, as in \cite{shervashidze2009fast}, ensuring two nodes at iteration $h+1$ have the same label if and only if their label and those of their neighbors at iteration $h$ are the same. Throughout the rest of paper, we denote the set of WL labels at iteration $h$ as $\Sigma^{h}$, for $h=1,...,H$, and the set of all WL labels as $\Sigma=\cup_{h=1}^{H}\Sigma^{h}$.
The WL labeling scheme is illustrated in Figure \ref{graphkernels} (a).

Based on WL labeling scheme, a WL node embedding scheme
was proposed to generate node embeddings from node labels of the graphs \cite{shervashidze2009fast}.
\begin{defn}
(WL based node embedding scheme, \cite{shervashidze2009fast}). Let $G=(V,E)$ and let $H$ be the number of
WL iterations. For every $h\in \{1,...,H\}$, 
we define the node embedding $x^{h}(v)$ of a node $v\in V$ and graph embedding $\textbf{X}^{h}(G)$ of $G$ at iteration $h$ as follows:
\begin{equation}
    x^{h}(v) = \ell^{h}(v) \text{, }
    \textbf{X}^{h}(G)=[x^{h}(v_{1}),...,x^{h}(v_{n_{V}})]^{T}
\end{equation}
Then the graph embedding of $G$ can be defined as:
\begin{align*}
    f^{H}(G): \mathcal{G}&\mapsto \Sigma^{n_{V}\times H}\\
    G &\mapsto \left[ \textbf{X}^{1}(G)^{T},...,\textbf{X}^{H}(G)^{T} \right]^{T}
\end{align*}
\end{defn}
\noindent
where $n_{V}$ is the number of nodes in $G$.
With the WL node embedding scheme above, we are ready to introduce some notations which will be used throughout
the rest of paper.\\

\noindent
\textbf{Notations:} Let $D_{h}^{\text{Ham}}(f^{h}(G), f^{h}(G^\prime))$ be the Hamming distance matrix
where each entry is the \emph{normalized Hamming distance} between the corresponding node embeddings of $G$ and $G^\prime$ at iteration $h$, defined as:
\begin{equation}
\label{nHamming}
d_{h}^{\text{Ham}}(u,v) = \frac{1}{h}\sum_{i=1}^{h}d_{i}^{\text{disc}}(u,v) \text{,  } d_{i}^{\text{disc}}(u,v)=
\begin{cases}
  0 & \text{if  } u_{i}=v_{i} \\
  1 & \text{otherwise}
\end{cases}
\end{equation}
where $u$ and $v$ denote two node embeddings, $d_{h}^{\text{Ham}}(u,v)$ denotes the normalized Hamming distance between $u$ and $v$ at iteration $h$ and $d_{i}^{\text{disc}}(u,v)$ denotes the \emph{discrete distance} between $u$ and $v$ at iteration $i$.
Similarly, let $D^{\text{Disc}}_{h}(\textbf{X}^{h}(G), \textbf{X}^{h}(G^\prime))$ be the discrete distance matrix where
each entry is the discrete distance between the corresponding node embeddings of $G$ and $G^\prime$ at iteration $h$. It is easy to see that $[D^{\text{Ham}}_{h}]_{ij}\in [0,1]$
and $[D^{\text{Disc}}_{h}]_{ij}\in \{0,1\}$. We also define a
base kernel which corresponds to the averaged number of feature shared by 
two node embeddings as: 
\begin{equation}
\label{Eqn:basekernel}
k_{h}(u,v) = \frac{1}{h}\sum_{i=1}^{h} \mathbb{1}(u_{i}=v_{i})
\end{equation}
It is easy to see that $k_{h}(u,v)=1- d_{h}^{\text{Ham}}(u,v)$. Thanks to WL node embedding scheme, a graph can be represented as a point cloud or a set of node embeddings. Measuring the similarity (or dissimilarity) between two graphs boils down to measuring the similarity (or dissimilarity) between two sets of embeddings. Each node embedding captures information about the neighborhood of the corresponding node (or a rooted \textit{subtree pattern}). In the following subsections, we will review different WL based graph kernels derived from different ways of comparing their sets of WL node embeddings.
\subsection{WL subtree kernels}
\label{WLKernel}
WL subtree kernels \cite{shervashidze2009fast} simply employ the aforementioned WL labeling scheme to extract subtree patterns, which represent the neighborhood of each node in the graph up to a given distance (or number of hops $H$). Essentially WL subtree kernel counts the number
of common WL labels. In the context of WL node embedding scheme,
it can be computed by summing all pairwise similarities between node embeddings of two graphs. 
More formally, for two graphs $G$ and $G^\prime$ with two sets of node embeddings 
$f^{H}(G)$ and $f^{H}(G^\prime)$, respectively, the WL subtree kernel value between them can be defined as:
\begin{equation}
\textbf{K}^{\text{WL}}(G,G^\prime) = \sum_{\textbf{x}\in f^{H}(G)}\sum_{\textbf{y}\in f^{H}(G^\prime)}k_{H}(\textbf{x},\textbf{y})
\end{equation}
It is obvious that as the base kernel $k_{H}(\textbf{x},\textbf{y})$ (defined in Eq. (\ref{Eqn:basekernel})) is equal to the number of WL labels (subtree patterns) shared by two node embeddings, the kernel $\textbf{K}^{\text{WL}}(G,G^\prime)$ is equal to the total number of WL subtree patterns shared by two graphs.
\subsection{WL-based optimal assignment kernels}
\label{OAKernel}
Optimal assignments are natural measures of similarity between two sets of points.  In particular, for two sets of points, the goal is to assign one point of one set to another point of the other set (one-to-one correspondence) such that the sum of similarities between  assigned points is maximized. Finding such an optimal alignment or bijection is also known as the well-studied assignment problem in combinatorial optimization \cite{munkres1957algorithms}. However, a challenge is how to design a valid kernel based on optimal assignments.

Kriege \textit{et al} \cite{kriege2016valid} introduced a restricted class of kernels, called \emph{strong kernels}, that guarantees the construction of valid optimal assignment based kernels. An important result is that the strong kernels give rise to hierarchies defined on the domain of kernels. Based on this, the authors proposed WL-based optimal assignment (WL-OA) kernels with the WL node embedding scheme. WL-OA kernels employ the base kernel, defined in Eq. (\ref{Eqn:basekernel}), which satisfies the requirement of being a strong kernel as the sequence of refined WL labels $(\ell_{0},\ell_{1},...,\ell_{H})$ gives rise to a family of nested subsets, which can be represented by a hierarchy. Consequently, WL-OA kernels are valid. Formally, given two graphs $G$ and $G^\prime$ with two sets of node embeddings $f^{H}(G)$ and $f^{H}(G^\prime)$, respectively, the WL-OA kernel value between them is defined as:
\begin{equation}
\label{Eqn:OAKernel}
\textbf{K}^{\text{OA}}(G,G^\prime) = \max_{B\in \mathcal{B}(f^{H}(G),f^{H}(G^\prime))}\sum_{(x,y)\in B}k_{H}(\textbf{x},\textbf{y})
\end{equation}
where $\mathcal{B}(f^{H}(G),f^{H}(G^\prime))$ is the set of all bijections between two sets $f^{H}(G)$ and $f^{H}(G^\prime)$. 
To apply this kernel to graphs of different number of nodes, we can fill up the graph with smaller number of nodes, says $f^{H}(G^\prime)$, by new node embeddings $z$ with $k(\textbf{x},\textbf{z})=0$ for all $\textbf{x}\in f^{H}(G)$ without changing the result.
It is worth noting that the WL-OA kernels take the similarities of aligned node embeddings into account, while the WL subtree kernels consider all pairwise similarities.

\begin{figure}[t]

	\includegraphics[width=1.0\columnwidth]{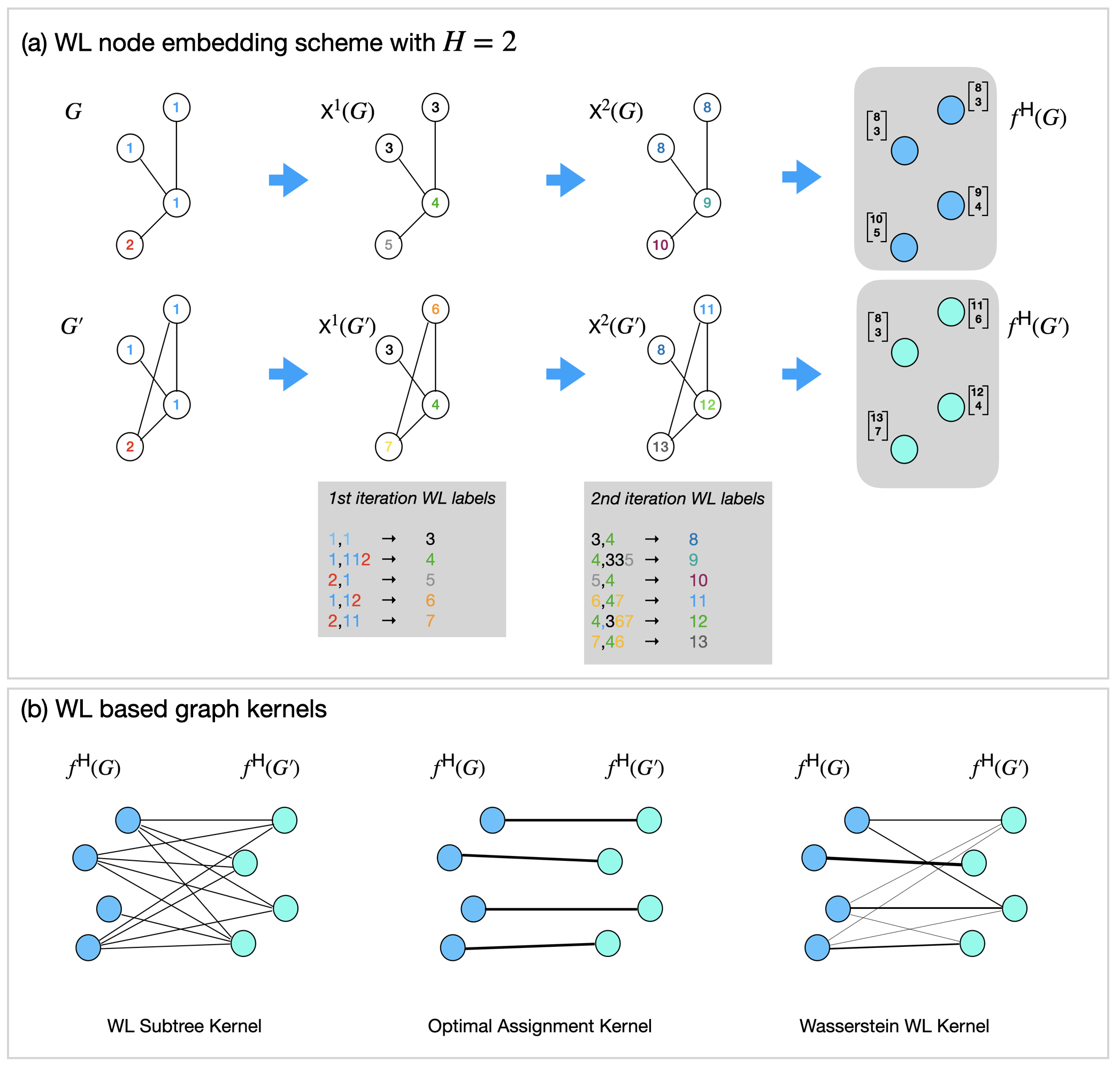}
	\caption{(a) Illustration of Weisfeiler-Lehman (WL) node embedding scheme with two iterations ($H=2$). (b) Illustration of different WL based graph kernels: WL Subtree kernel is computed by considering all pairwise similarities between node embeddings of two graphs; WL optimal assignment kernel takes  the similarities of only aligned node embeddings into account; For Wasserstein WL kernel, one node embedding of one graph can be coupled with multiple node embeddings of the other graph for computing the kernel.} 
	\label{fig:comparison}
\label{graphkernels}
\end{figure}

\subsection{Wasserstein WL kernels}
\label{WWLKernel}
Optimal transport (OT), also known as Wasserstein distance function \cite{villani2008optimal}, has gained much
attraction in machine learning community as a powerful tool for the comparsion
of two probability distributions. The naive computation of this distance between
two discrete measures, e.g. point clouds, involves solving transport problem. Formally, let $X=\{\textbf{x}_{1},...,\textbf{x}_{m}\}$ and $Y=\{\textbf{y}_{1},...,\textbf{y}_{n}\}$ be two sets of points, where $m$ and $n$ denote the size of two sets $X$ and $Y$, respectively; $\textbf{p}\in\mathbb{R}^{m}$ and $\textbf{q}\in\mathbb{R}^{n}$ are two discrete probability distributions over $X$ and $Y$, respectively. We use $d_{ij}$ to denote the distance between $\textbf{x}_{i}$ and $\textbf{y}_{j}$, e.g. the squared Euclidean distance $d_{ij}=\left\Vert \textbf{x}_{i}-\textbf{y}_{j}\right\Vert^{2}_{2}$. The Wasserstein distance is formulated as a linear program over the \emph{transportation matrix} (or joint probability) $\textbf{P}\in\mathbb{R}^{m\times n}$:

\begin{equation}
\label{Eqn:OT}
\begin{aligned}
\mathcal{W}_{1}(X,Y, d) =& \min_{\textbf{P}_{ij}}\sum_{i=1}^{m}\sum_{j=1}^{n}\textbf{P}_{ij}d_{ij}\\
\text{subject to } & \sum_{i=1}^{m}\textbf{P}_{ij}=\textbf{q}_{j}\text{, } \forall j\in [1,n] \\
& \sum_{j=1}^{n}\textbf{P}_{ij}=\textbf{p}_{i}\text{, } \forall i\in [1,m] \\
\end{aligned}
\end{equation}

With WL node embedding scheme to generate node embeddings for graphs, Togninalli \textit{et al} \cite{togninalli2019wasserstein} evaluated the pairwise Wasserstein distance between graphs with the normalised Hamming (\ref{nHamming}) as the ground distance. Then, Wasserstein WL (WWL) kernel is defined as an instance of Laplacian kernels, see Eq. (\ref{Eqn:WWLkernel}).
\begin{equation}
\label{Eqn:WWLkernel}
\begin{aligned}
\textbf{D}^{\text{WWL}}(G,G^\prime)&=\mathcal{W}_{1}(f^{H}(G),f^{H}(G^\prime), d^{\text{Ham}}_{H})\\
\textbf{K}^{\text{WWL}}(G,G^\prime)&=e^{-\gamma\textbf{D}^{\text{WWL}}(G,G^\prime)}
\end{aligned}
\end{equation}
where $\gamma$ is a hyperparameter.

In general cases, it is not necessarily possible to derive a valid kernel
from the Wasserstein distance. However, thanks to the special property of WL labels and normalized Hamming distance (\ref{nHamming}), $\textbf{D}^{\text{WWL}}$ was shown to be conditionally negative definite (CND), resulting in the validity of $\textbf{K}^{\text{WWL}}$, by proving the following lemma (see \cite{togninalli2019wasserstein} for its proof):
\begin{lem}
\label{Lemma1}
If a transportation matrix $P^{H}$ is optimal solution of (\ref{Eqn:OT}) with the ground distance $d_{H}^{\text{Ham}}$ (\ref{nHamming}) between node embeddings at iteration $H$, then we have the two following claims:
\begin{enumerate}
    \item $P^{H}$ is also optimal solution of (\ref{Eqn:OT}) with the discrete distance $d_{H}^{\text{Disc}}$ between $H$-iteration values.
    \item $P^{H}$ is also optimal solution of (\ref{Eqn:OT}) with the normalised Hamming distance $d_{H-1}^{\text{Ham}}$ between node embeddings at iteration $H-1$.
\end{enumerate}
\end{lem}

\noindent
Let $P^{*}$ be the optimal solution of (\ref{Eqn:OT}) for $D^{\text{Ham}}_{H}$. From the above lemma, it is also the optimal solution for $D^{\text{Disc}}_{h}$, $h=1,..,H$. 
The Wasserstein distance between two graphs $G$ and $G^\prime$ in (\ref{Eqn:WWLkernel}) can be simplified as follows:
\begin{equation}
\label{Eqn:simpleWWLDist}
    \textbf{D}^{WWL}(G,G^\prime) = \frac{1}{H}\sum_{h=1}^{H}\mathcal{W}_{1}(\textbf{X}^{h}(G),\textbf{X}^{h}(G^\prime), d^{\text{Disc}}_{h})
\end{equation}

The Eq. (\ref{Eqn:simpleWWLDist}) is a sum of OT distances with the discrete distances as ground metrics, which are CND. Therefore, the sum is also CND, leading to the validity of the similarity matrix $\textbf{K}^{\text{WWL}}$.

\section{Incorporating subtree pattern importance into WL based graph kernels}
One of the main limitations of kernels is the decoupling of data representation and learning process, that is, the kernel must be predefined prior to learning, leading to limited predictive performance. Furthermore, in prediction tasks for molecular data, 
the output might be determined by the presence of a few important substructures, while these kernels contain all substructures with equal weights. Motivated by this drawback, in this paper we address the problem of incorporating subtree pattern weights for WWL kernel \cite{togninalli2019wasserstein}. To this end, we aim to learn new kernels from a parametric form of Wasserstein distance taking into account subtree pattern weights (\ref{Eqn:simpleWWLDist}), and learn these weights from data optimally for the task.
\subsection{Parametric form of Wasserstein distance with subtree pattern weights}
To derive a parametric form of the distance function (\ref{Eqn:simpleWWLDist}), we rely on the following simple
observation:
\begin{lem}
Let $S$ be a set of elements, $X=\{\textbf{x}_{1},...,\textbf{x}_{m}\}$ and $Y = \{\textbf{y}_{1},...,\textbf{y}_{n}\}$ ($X,Y \subseteq S$) be two multiset of $S$ of $m$ and $n$ samples, respectively, the Wasserstein distance between them with the discrete distance as the ground metric is determined by:
\begin{equation}
    \mathcal{W}_{1}(X,Y) = 1 - \sum_{v\in S} \min(\mu_{X}(v), \mu_{Y}(v))
\end{equation}
where $\mu_{X}(v)$ denotes the mass density function of the multiset $X$ with $v\in S$.
\end{lem}
\noindent
Applying this lemma to Eq. (\ref{Eqn:simpleWWLDist}), we have:
\begin{equation}
\label{Eqn:newform}
\textbf{D}^{\text{WWL}}(G, G^\prime)=1- 
\frac{1}{H}\sum_{h=1}^{H}\sum_{v\in \Sigma^{h}}\min(\mu_{\textbf{X}_{h}(G)}(v), \mu_{\textbf{X}_{h}(G^\prime)}(v))
\end{equation}

Our idea is to give each substructure or WL label $v$ a nonnegative weight $w_{v}\in \mathbb{R}_{\geq 0}$ for its importance to the problem, so the parametric form of Eq. (\ref{Eqn:newform}) is defined as follows:
\begin{equation}
b- 
\frac{1}{H}\sum_{h=1}^{H}\langle\textbf{w}_{h},\textbf{z}_{h}\left(G,G^\prime\right)\rangle
\end{equation}
where $\textbf{w}_{h}$, $\textbf{z}_{h}\left(G,G^\prime\right)\in\mathbb{R}^{|\Sigma^{h}|}$ are the vectors of entries $w_{v}$ and $\min(\mu_{\textbf{X}_{h}(G)}(v), \mu_{\textbf{X}_{h}(G^\prime)}(v))$, respectively, for $v\in \Sigma^{h}$;
$b$ is a constant to ensure that the value of parametric function is nonnegative. In vector form, this can be expressed as:
\begin{equation}
\label{Eqn:linearform}
d_{\textbf{W}}(G,G^\prime)=b- 
\langle\textbf{W}, \textbf{Z}\left(G,G^\prime\right)\rangle
\end{equation}
where 
$
\begin{cases}
  \textbf{Z}\left(G,G^\prime\right)=\frac{1}{H}\left[ \textbf{z}_{1}\left(G,G^\prime\right)^{T},...,\textbf{z}_{H}\left(G,G^\prime\right)^{T} \right]^{T} \\
  \textbf{W}=\left[ \textbf{w}_{1}^{T},...,\textbf{w}_{H}^{T} \right]^{T}
\end{cases}$\\
The parametric form (\ref{Eqn:linearform}) is a linear function with respect to the parameter vector $\textbf{W}\in \mathbb{R}^{d}$ ($d=|\Sigma^{1}|+...+|\Sigma^{H}|$) and $\textbf{Z}\left(G,G^\prime\right)$ is considered as a feature vector
of a pair of graphs $G$ and $G^\prime$. Once the parameters are estimated, we can derive a similarity matrix through the Laplacian kernel as in Eq. (\ref{Eqn:WWLkernel}). More importantly, as $\textbf{W}$ is nonnegative, it is easy to see that $d_{\textbf{W}}$ is a CND function, and thus the derived similarity matrix is valid.
\subsection{Formulation of learning subtree pattern weights $\textbf{W}$}
We aim to learn the parameters $W$ in Eq. (\ref{Eqn:linearform}) using the notions of metric learning \cite{kulis2012metric}. That is two input graphs
with the same labels are encouraged to be closer while the two with different labels become far away from each other. In other words, within class distances should be small, while between class distances should be large. For this purpose, as a loss function for a graph pair $g$ and $g^\prime$, we can use the following two hinge loss function: $\max(0, \alpha_{1}-d_{\textbf{W}}(g, g^\prime))$ if $g$ and $g^\prime$ are with different labels and $\max(0, d_{\textbf{W}}(g, g^\prime)-\alpha_{2})$ otherwise, for learning subtree pattern weights, where $\alpha_{1}$ and $\alpha_{2}$ are constants ($\alpha_{1}\geq\alpha_{2}$). The former yields a penalty if $g$ and $g^\prime$ of different labels are closer than $\alpha_{1}$ while the latter yields a penalty when $g$ and $g^\prime$ of the same label are more distant than $\alpha_{2}$. 
Instead of using these functions in the optimization problem, we use their smooth versions: $V_{1}$ and $V_{2}$ (see Figure \ref{Fig:smoothfunctions}), which offer useful properties for deriving a generalization bound for the problem in Section 4. The derivation of these functions is based on the connection between the strong convexity of a function and Lipschitz continuous gradient of its Fenchel dual (see more details in \cite{nesterov2005smooth}). 

More concretely, let $D_{n} = \{z_{1}=(g_{1},y_{1}),...,z_{n}=(g_{n},y_{n})\}$ where $g_{i}\in \mathcal{G}$ and $y_{i}\in \mathcal{Y}=\{-1,1\}$, for $i=1,...,n$, we formulate a constrained minimization problem for learning subtree pattern weights as follows:

\begin{mini}|l|
  {\textbf{W}}{\frac{1}{n^{2}}\sum_{i=1}^{n}\sum_{j=1}^{n}\ell(\textbf{W}, z_{i}, z_{j})}{}{}
  \addConstraint{\textbf{W}\in \mathcal{C}}
  \label{Eqn:ell}
\end{mini}
where $\mathcal{C}=\{\textbf{W}\in\mathbb{R}^{d}:\textbf{W}=\left[ \textbf{w}_{1}^{T},...,\textbf{w}_{H}^{T} \right]^{T}, ||\textbf{w}_{h}-\textbf{c}_{h}||_{2}\leq \epsilon_{h}, 1\leq h \leq H \}$ ($\textbf{c}_{h}$ and $\epsilon_{h}$ are constant vectors and scalars); 
$\ell$ is a continuously differentiable function and defined as follows:

$\ell(\textbf{W}, z_{i}, z_{j})=\begin{cases}
  V_{1}(\textbf{W}, g_{i}, g_{j}) & \text{ if } y_{i}\neq y_{j}\\
  V_{2}(\textbf{W}, g_{i}, g_{j}) & \text{ otherwise}
\end{cases}$

\begin{equation}
\label{Eqn:V1}
    V_{1}(\textbf{W}, g_{i}, g_{j})=
    \begin{cases}
      0 & \text{ if }  d_{\textbf{W}}(g_{i}, g_{j}) \geq \alpha_{1}\\
      \alpha_{1} - \frac{\sigma}{2} -  d_{\textbf{W}}(g_{i}, g_{j})& \text{ if }  d_{\textbf{W}}(g_{i}, g_{j}) \leq \alpha_{1}-\sigma\\
      \frac{1}{2\sigma}\left( d_{\textbf{W}}(g_{i}, g_{j}) - \alpha_{1} \right)^{2} & \text{ if } \alpha_{1} - \sigma < d_{\textbf{W}}(g_{i}, g_{j}) < \alpha_{1}
    \end{cases}
\end{equation}

\begin{equation}
\label{Eqn:V2}
    V_{2}(\textbf{W}, g_{i}, g_{j})=
    \begin{cases}
      0 & \text{ if }  d_{\textbf{W}}(g_{i}, g_{j}) \leq \alpha_{2}\\
      d_{\textbf{W}}(g_{i}, g_{j}) - \alpha_{2} - \frac{\sigma}{2}& \text{ if }  d_{\textbf{W}}(g_{i}, g_{j}) \geq \alpha_{2}+\sigma\\
      \frac{1}{2\sigma}\left( d_{\textbf{W}}(g_{i}, g_{j}) - \alpha_{2} \right)^{2} & \text{ if } \alpha_{2} < d_{\textbf{W}}(g_{i}, g_{j}) < \alpha_{2}+\sigma
    \end{cases}
\end{equation}
and $\alpha_{1}$, $\alpha_{2} \left( \alpha_{1} \geq \alpha_{2} \right)$ and $\sigma$ are constants; $d_{\textbf{W}}(g_{i}, g_{j})$ is calculated from $\textbf{Z}\left(g_{i}, g_{j}\right)$ as in Eq. (\ref{Eqn:linearform}). To reduce notation, we use $\textbf{Z}_{i,j}$ rather than $\textbf{Z}\left(g_{i}, g_{j}\right)$ in the rest of the paper.
\begin{lem}
\label{propertiesofell}
Let $\beta=\max_{1\leq i<j\leq n} \left\Vert \textbf{Z}_{ij}\right\Vert_{2}$, $L=\frac{\beta^{2}}{\sigma}$,  and $M=\max(b-\frac{\sigma}{2}-\alpha_{2}, \alpha_{1}-\frac{\sigma}{2})$, the loss function $\ell$ defined in problem (\ref{Eqn:ell}) is $L$-lipschitz, $\beta$-smooth and upper bounded by $M$.
\end{lem}
\begin{proof} First, it is obvious to see that as $0\leq d_{\textbf{W}}\leq b$, we have the following bounds: $0\leq V_{1}\leq \alpha_{1}-\frac{\sigma}{2}$ and $0\leq V_{2}\leq b-\frac{\sigma}{2}-\alpha_{2}$. Therefore, $\ell$ is upper bounded by $M$.

We derive the first and second order derivatives of the smooth function $V_{1}$ as follows:
\begin{equation}
\label{Eqn:1ordV1}
    \nabla V_{1}(\textbf{W}, g_{i}, g_{j})=
    \begin{cases}
      0 & \text{ if }  d_{\textbf{W}}(g_{i}, g_{j}) \geq \alpha_{1}\\
     \textbf{Z}_{ij}& \text{ if }  d_{\textbf{W}}(g_{i}, g_{j}) \leq \alpha_{1}-\sigma\\
      \frac{1}{\sigma}\left( \alpha_{1}- d_{\textbf{W}}(g_{i}, g_{j})  \right)\textbf{Z}_{ij} & \text{ if } \alpha_{1} - \sigma < d_{\textbf{W}}(g_{i}, g_{j}) < \alpha_{1}
    \end{cases}
\end{equation}
\begin{equation}
\label{Eqn:2ordV1}
    \nabla^{2} V_{1}(\textbf{W}, g_{i}, g_{j})=
    \begin{cases}
      0 & \text{ if }  d_{\textbf{W}}(g_{i}, g_{j}) \geq \alpha_{1}\\
      0& \text{ if }  d_{\textbf{W}}(g_{i}, g_{j}) \leq \alpha_{1}-\sigma\\
     \frac{1}{\sigma}\textbf{Z}_{ij}\textbf{Z}_{ij}^{T} & \text{ if } \alpha_{1} - \sigma < d_{\textbf{W}}(g_{i}, g_{j}) < \alpha_{1}
    \end{cases}
\end{equation}
In order to prove the function $V_{1}$ is $L$-Lipschitz and $\beta$-smooth, it is sufficient to show that the norm of its derivative is always less than $L$: $\left\Vert\nabla V_{1}(\textbf{W},g_{i}, g_{j})\right\Vert_{2}\leq L$ and the spectral norm (or the maximum eigen value) of its second order derivative is always less than $\beta$: $||\nabla^{2} V_{1}(\textbf{W}, g_{i}, g_{j})||\leq \beta\text{, } \forall g_{i}, g_{j}\in\mathcal{G}$. Indeed, from Eq. (\ref{Eqn:1ordV1}) and (\ref{Eqn:2ordV1}), we have: $\left\Vert V_{1}(\textbf{W},g_{i},g_{j})\right\Vert_{2}\leq \left\Vert \textbf{Z}_{ij}\right\Vert_{2}$ and $ \left\Vert \nabla^{2}V_{1}(\textbf{W}, g_{i}, g_{j})\right\Vert\leq \frac{1}{\sigma}\left\Vert \textbf{Z}_{ij}\textbf{Z}_{ij}^{T}\right\Vert=\frac{1}{\sigma}\left\Vert \textbf{Z}_{ij}\right\Vert_{2}^{2}$. Also, we can bound $\left\Vert \textbf{Z}_{ij}\right\Vert_{2}$ by the  inequality $\left\Vert \textbf{Z}_{ij}\right\Vert_{2}\leq \beta$. Thus $V_{1}$ is $\beta$-smooth and $L$-Lipschitz. Similarly, we can also show that $V_{2}$ is $\beta$-smooth and $L$-Lipschitz. The lemma is proven.
\end{proof}

\begin{figure}[t]
    
	\includegraphics[width=1.0\columnwidth]{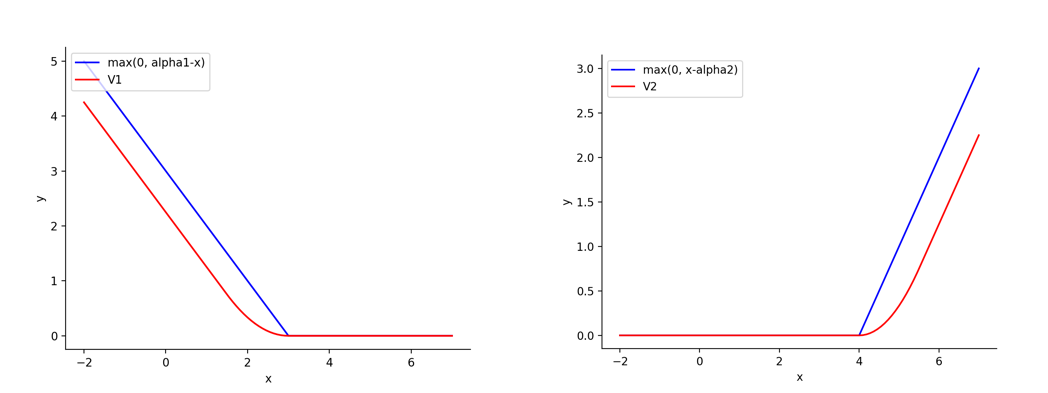}
	\caption{Illustration of Hinge loss $\max(0, \alpha_{1}-x)$ and $\max(0, x-\alpha_{2})$; and their smooth versions: $V_{1}$ (left) and $V_{2}$ (right), respectively, with $\alpha_{1}=3, \alpha_{2}=4$ and $\sigma=1.5$.} 
	\label{Fig:smoothfunctions}
\end{figure}

\subsection{A stochastic learning algorithm for constrained optimization}

The constrained optimization problem (\ref{Eqn:ell}) is convex and thus guarantees to find its global optimum. Standard methods such as projected gradient descent can be used to solve the problem (\ref{Eqn:ell}). However, for large scale data sets, solving the problem (\ref{Eqn:ell}), involving $n^{2}$ terms with $d$ parameters, might be computationally expensive. For instance, the data set PROTEIN (see Table 1) has $n^{2}> 10^6$ pairs of examples and the weight vector size of $d> 10^5$ with the number of WL iterations $H=5$. In this subsection, we present an efficient stochastic learning algorithm for dealing with this issue.

\begin{algorithm}[H]
\label{pSGD}
\SetAlgoLined
\KwIn{$D_{n}=\{z_{1},...,z_{n}\}$, $c$, $\mu$: learning rate and $T$:\#maxIters}
\KwOut{\text{solution }$\textbf{W}^{\star}$}
 $\textbf{W}^{(0)}=\textbf{1}_{d}$, $t=0$\;
 \While{$t\leq T$}{
  randomly pick two examples $z=(g,y)$ and $z^\prime=(g^\prime,y^\prime)$ from $D_{n}$\;
  \eIf{$y = y^\prime$}{
   $\mathrm{grad}^{(t)}=\nabla V_{1}(\textbf{W}^{(t)},g,g^\prime)$\;
   }{
   $\mathrm{grad}^{(t)}=\nabla V_{2}(\textbf{W}^{(t)},g,g^\prime)$\;
  }
  $\textbf{W}^{(t+1)}=\mathrm{proj}_{\mathcal{C}}[\textbf{W}^{(t)}-\mu \mathrm{grad}^{(t)}]$\;
 }
 $\textbf{W}^{\star}=\textbf{W}^{(T)}$\;
 \caption{A stochastic algorithm for learning $\textbf{W}$}
\end{algorithm}

Let $\textbf{W}^{(t)}$ denote the weight at iteration $t$. The weight is initialized by a vector of ones: $\textbf{W}^{(0)}=\textbf{1}_{d}$, which is also the case of WWL kernel without learning subtree pattern importance.
At each iteration $t$, we randomly pick up a pair of examples ($z = (g,y),z^\prime=(g^\prime,y^\prime)$) from the training data set $D_{n}$ and compute the gradient $\mathrm{grad}^{(t)}$ corresponding to this pair. In fact this step can be done efficiently due to the sparsity of the feature vector $\textbf{Z}\left(g,g^\prime\right)$ in Eq. (\ref{Eqn:linearform}). Then, we update the current solution $\textbf{W}^{(t)}$ to $\textbf{W}^{(t+1)}$ by the following rule:
\begin{equation*}
    \textbf{W}^{(t+1)}=\mathrm{proj}_{\mathcal{C}}[\textbf{W}^{(t)}-\mu  \mathrm{grad}^{(t)}]
\end{equation*}
where $\mathrm{proj}_{\mathcal{C}}[\textbf{W}]_{h}=
\begin{cases}
  \textbf{w}_{h} & \text{if } ||\textbf{w}_{h}-\textbf{c}_{h}||_{2}\leq \epsilon_{h}\\
  \frac{\epsilon_{h}}{||\textbf{w}_{h}-\textbf{c}_{h}||_{2}}(\textbf{w}_{h}-\textbf{c}_{h})+\textbf{c}_{h} & \text{otherwise}
\end{cases}$
maps a point $\textbf{w}_{h}$ ($1\leq h \leq H$) back to the bounded feasible region. The procedure is illustrated in Algorithm \ref{pSGD}.

\section{Theoretical Guarantees: A Bound on Generalization Gap}
In this section, we provide a bound on the generalization gap of the proposed stochastic learning algorithm for solving the problem (\ref{Eqn:ell}). The gap is defined as the expected difference
between the generalization error $\textit{R}(.)$ and empirical error 
$\textit{R}_{D_{n}}(\cdot)$. In order to derive the generalization bound, we  first provide basic setup and notations; then prove that our learning algorithm has a uniform stability, which is established in Theorem \ref{theorem:uniformstability}
using Lemma \ref{propertiesofell} and Lemma \ref{boundonparams}; finally derive our generalization bound, which is established in Theorem \ref{maintheorem} using the McDiarmid inequality (see Theorem \ref{McDiarmid}).
\subsection{Basic setup and notations}
\textbf{Generalization Error}. Let $\textbf{W}_{n}$ be the parameters of the parametric function (\ref{Eqn:linearform}) obtained by training on the data set
$D_{n}$ using Algorithm \ref{pSGD}. Then the generalization error (or risk) $\textit{R}(\textbf{W}_{n})$ with a loss
function $\ell$ is defined as:
\begin{equation*}
    \textit{R}(\textbf{W}_{n})=\textbf{E}_{z,z^\prime}[\ell(\textbf{W}_{n}, z, z^\prime)]
\end{equation*}
where $\textbf{E}_{z,z^\prime}\left[\ell(\cdot,\cdot,\cdot)\right]$ denotes the expectation of function $\ell$ when $z$ and $z^\prime$ are sampled according the distribution $\mathcal{P}$.\\ 
\textbf{Empirical Error}. The empirical error $\textit{R}_{D_{n}}(\textbf{W}_{n})$ is defined on the training data set $D_{n}$ as :
\begin{equation*}
    \textit{R}_{D_{n}}(\textbf{W}_{n})=\frac{1}{n^{2}}\sum_{z_{i}\in D_{n}}\sum_{z_{j}\in D_{n}}\ell(\textbf{W}_{n}, z_{i},z_{j})
\end{equation*}
\textbf{Expected Generalization Gap}. As Algorithm \ref{pSGD} is based on a randomized procedure, we use the definition of the expected
generalization gap as follows:
\begin{equation*}
\label{expectedgap}
    \mathbb{K}_{n}= \textbf{E}_{\text{SGD}}[\textit{R}(\textbf{W}_{n})-\textit{R}_{D_{n}}(\textbf{W}_{n})]
\end{equation*}
where $\textbf{E}_{\text{SGD}}$ denotes the expectation taken over the inherent randomness of the stochastic algorithm.

\subsection{Uniform stability of the stochastic learning algorithm}
Intuitively, a learning algorithm is said to have a uniform stability if its output is stable under a small modification of the training data set. For a randomized learning algorithm, the uniform stability property is defined as follows:
\begin{defn}
[Uniform Stability of the randomized algorithm] A randomized algorithm $\mathbb{A}$ is $\beta_{n}$-uniformly stable with respect to a loss function
$\ell$, if the following inequality holds:
\begin{equation*}
   \forall (D_{n}, k)\text{,  } \sup_{z, z^\prime}|\textbf{E}_{\text{SGD}}[\ell(\textbf{W}_{n}, z, z^\prime)]-\textbf{E}_{\text{SGD}}[\ell(\textbf{W}_{n,k}, z, z^\prime)]|\leq \beta_{n}
\end{equation*}
where $D_{n,k}$ is the new data set obtained from $D_{n}$ by replacing $z_{k}\in D_{n}$ with a new example $\hat{z}_{k}$ sampled from $\mathcal{P}$;  $\textbf{W}_{n}$ and $\textbf{W}_{n,k}$ are the outputs of $\mathbb{A}$ trained on two data sets $D_{n}$ and $D_{n,k}$, respectively.
\end{defn}
In order to prove that Algorithm \ref{pSGD} has the uniform stability property, we need the following lemma (its proof is placed in the appendix section):
\begin{lem}
\label{boundonparams}
Let the loss function $\ell$ defined in the problem (\ref{Eqn:ell}) be $\beta$-smooth and 
$L$-Lipschitz; $\textbf{W}_{n}^{(T)}$ and $\textbf{W}_{n,k}^{(T)}$ be the parameters of the parametric form (\ref{Eqn:linearform}) trained on $D_{n}$ and $D_{n,k}$, respectively, using Algorithm \ref{pSGD} with the number of iterations $T$ and learning rate $\mu$. Then, the expected difference in the model parameters is upper bounded by:
\begin{equation}
    \textbf{E}_{\text{SGD}}\left[\left\Vert \textbf{W}_{n}^{(T)}-\textbf{W}_{n,k}^{(T)}\right\Vert_{2}\right]\leq \frac{4}{n}\mu T L
\end{equation}
\end{lem}

Using Lemma \ref{boundonparams} and $L$-Lipschitz property of function $\ell$ (see Lemma \ref{propertiesofell}), we can now prove the stability of Algorithm \ref{pSGD}.
\begin{thm}
\label{theorem:uniformstability}
[Uniform Stability of Algorithm \ref{pSGD}] Let the loss function $\ell$ defined in the problem (\ref{Eqn:ell}) be $\beta$-smooth and 
$L$-Lipschitz. Then Algorithm \ref{pSGD} with the fixed learning rate $\mu$ is $k_{n}$-uniformly stable where $k_{n}=\frac{4}{n} \mu T L^{2} $.
\end{thm}
\begin{proof}
We have the following inequalities:
\begin{equation}
    |\textbf{E}_{\text{SGD}}[\ell(\textbf{W}_{n}, z, z^\prime)]-\textbf{E}_{\text{SGD}}[\ell(\textbf{W}_{n,k}, z, z^\prime)]|\leq L\textbf{E}_{\text{SGD}}\left\Vert \textbf{W}_{n}-\textbf{W}_{n,k}\right\Vert_{2} \leq \frac{4}{n} \mu T L^{2}
\end{equation}
where the first and second inequalities are obtained by the $L$-Lipschitz property of $\ell$ and Lemma \ref{boundonparams}, respectively. This completes the proof.
\end{proof}

\subsection{Bound on generalization gap}
Using the property of uniform stability in the previous subsection,
we can derive the generalization bound which is done by the McDiarmid inequality \cite{mcdiarmid1989method}.
\begin{thm}
\label{McDiarmid}
[McDiarmid inequality \cite{mcdiarmid1989method}] Let $X_{1},X_{2},...,X_{n}$ be
$n$ independent random variables taking values in $\mathcal{X}$
and let $Z=f(X_{1},...,X_{n})$. If, for each $1\leq i\leq n$, there exists a constant $c_{i}$ such that\\

  $\sup_{x_{1},...,x_{n},x_{i}^\prime}|f(x_{1},...,x_{n})-f(x_{1},...,x_{i}^\prime,...,x_{n})|\leq c_{i}\text{, }\forall 1\leq i\leq n\text{,}$\\
  
  then for any $\epsilon>0$, $\text{Pr}[|Z-\textbf{E}[Z]|\geq\epsilon]\leq2\mathrm{exp}\left(\frac{-2\epsilon^{2}}{\sum_{i=1}^{n}c_{i}^{2}}\right)$
\end{thm}

\noindent
To derive the bound on $R(W_{n})$, we replace $Z$ by $K_{n}$ in Theorem \ref{McDiarmid} and bound $\textbf{E}_{\text{SGD}}\left[ \mathbb{K}_{n}\right]$ and $|\mathbb{K}_{n}-\mathbb{K}_{n,k}|$ which are established by the following lemmas (see their proofs in the appendix section).

\begin{lem}
\label{Lemma:BoundKn}
For the loss function satisfying a uniform stability in $k_{n}$, we have the following inequality:
\begin{equation}
    \textbf{E}_{D_{n}}[\mathbb{K}_{n}]\leq 2k_{n}
\end{equation}
\end{lem}

\begin{lem}
\label{Lemma:BoundKminusKn}
For the loss function satisfying a uniform stability in $k_{n}$ and upper bounded by $M$, we have the following inequality:
\begin{equation}
    |\mathbb{K}_{n}-\mathbb{K}_{n,k}|\leq 2 k_{n}+\frac{2M}{n}
\end{equation}
\end{lem}
\noindent
Now we can derive the generalization bound for $R(\textbf{W}_{n})$ in the following theorem:
\begin{thm}
\label{maintheorem}
Let $D_{n}$ be a training data set with $n$ samples, $\textbf{W}_{n}$ be the solution obtained by minimizing the optimization problem (\ref{Eqn:ell}) using Algorithm \ref{pSGD} with uniform
stability $k_{n}$. Then the following inequality holds for probability of at least $1-\delta$ $\left(0\leq\delta\leq1\right)$:
\begin{equation}
    \textbf{E}_{\text{SGD}}\left[ R(\textbf{W}_{n})-\textit{R}_{D_{n}}(\textbf{W}_{n}) \right] \leq
    2 k_{n} + (n k_{n} + M)\sqrt{\frac{2}{n}\mathrm{log}\frac{2}{\delta}}
\end{equation}
\end{thm}
\begin{proof}
Applying McDiarmid's concentration inequality (\ref{McDiarmid}) by replacing $Z$ with $\mathbb{K}_{n}$, we have:
\begin{align*}
    Pr(\mathbb{K}_{n} - \textbf{E}_{D_{n}}\left[ \mathbb{K}_{n}\right] \geq \epsilon)\leq 2\mathrm{exp}\left(\frac{-2\epsilon^{2}}{n\left(2 k_{n}+\frac{2M}{n}\right)^{2}}\right)
\end{align*}
By fixing $\delta=2\mathrm{exp}\left(\frac{-2\epsilon^{2}}{n\left(2 k_{n}+\frac{2M}{n}\right)^{2}}\right)$, we get $\epsilon=(n k_{n} + M)\sqrt{\frac{2}{n}\mathrm{log}\frac{1}{\delta}}$ which completes the proof of Theorem \ref{maintheorem}.
\end{proof}
The generalization bound is meaningful if the bound converge to 0 as $n\to \infty$. Our derived generalization bound converges to 0 as $k_{n}$ decays with $O(\frac{1}{n})$. This confirms  Algorithm \ref{pSGD} converges.

\section{Experiments}
In this section, we demonstrate the benefit of learning subtree pattern weights by experiments
on both synthetic and real-world data. We performed classification experiments using the C-SVM implementation LIBSVM \cite{chang2011libsvm}. The necessary parameters of SVM were selected by cross-validation on the training set. These are the regularization parameter $C \in \{10^{-3},10^{-2},..., 10^{2},10^{3}\}$ and kernel parameter $\gamma \in \{0.0001, 0.001,0.01\}$. 
For learning the weights $\textbf{W}$ of subtree patterns (or WL labels) in Algorithm \ref{pSGD}, the learning rate $\mu$ and maximum number of iterations $T$ were set as $0.0001$ and 500, respectively; $\epsilon_{h}\in \{0.1, 0.5, 1.0\}$ were selected by cross validation based on the training set and $\textbf{c}_{h}$ was fixed as a vector of ones, i.e. $\textbf{c}_{h}=\textbf{1}_{|\Sigma^{h}|}$ for $h=1,..,H$;
the hyperparameters $\alpha_{1}$, $\alpha_{2}$ and $\sigma$ were empirically determined as 1.0, 0.5 and 0.1, respectively.
All kernels were implemented in Python 3.0 and experiments were conducted on an Intel Core i9 at 2.3 Ghz with 64GB of RAM using a single processor only. The source code can be accessed through \url{https://github.com/haidnguyen0909/weightedWWL}.
\subsection{Synthetic data}
We designed eight substructures, shown in Figure \ref{Fig:synthetic_subs}, in which substructures indexed by 1, 2, 5 and 6 are assumed to be indicative to positive class (+1) as they have a pattern '1-0(-2)-0' in common. The others are indicative to negative class (-1). Our synthetic data set consists of \textit{eight groups of graphs, each corresponds to one of these eight substructures} by randomly adding noisy nodes and edges. We used groups corresponding substructures 1, 2, 3 and 4 as training data and the others as testing data. We constructed two kernels for graphs: 
WWL and the proposed method with number of WL iterations $H=2$, then used SVM for classification. We reported mean accuracy obtained by ten synthetic data sets generated in this way.

We observed that WWL obtained mean accuracies of 82.4\%, while the proposed method achieved significantly higher accuracy of 95\%. It is noted that the testing examples were confusing the classifier. For instance, the substructure 5 has the same similarity with substructures 2 (indicative to positive class) and 4 (indicative to negative class) according to the WWL kernel, making it hard for the classifier to distinguish graphs containing the substructure 5. In contrast, 
this confusion can be alleviated by learning subtree pattern weights. In particular, the pattern '1-0(-2)-0' present in the substructure 5 is assigned a high weight by the proposed method (see Figure \ref{Fig:synthetic_w}). Therefore, graphs generated from the substructure 5 are more likely to be classified as positive.

\begin{figure}[t]
    \centering
	\centerline{\includegraphics[width=0.8\columnwidth]{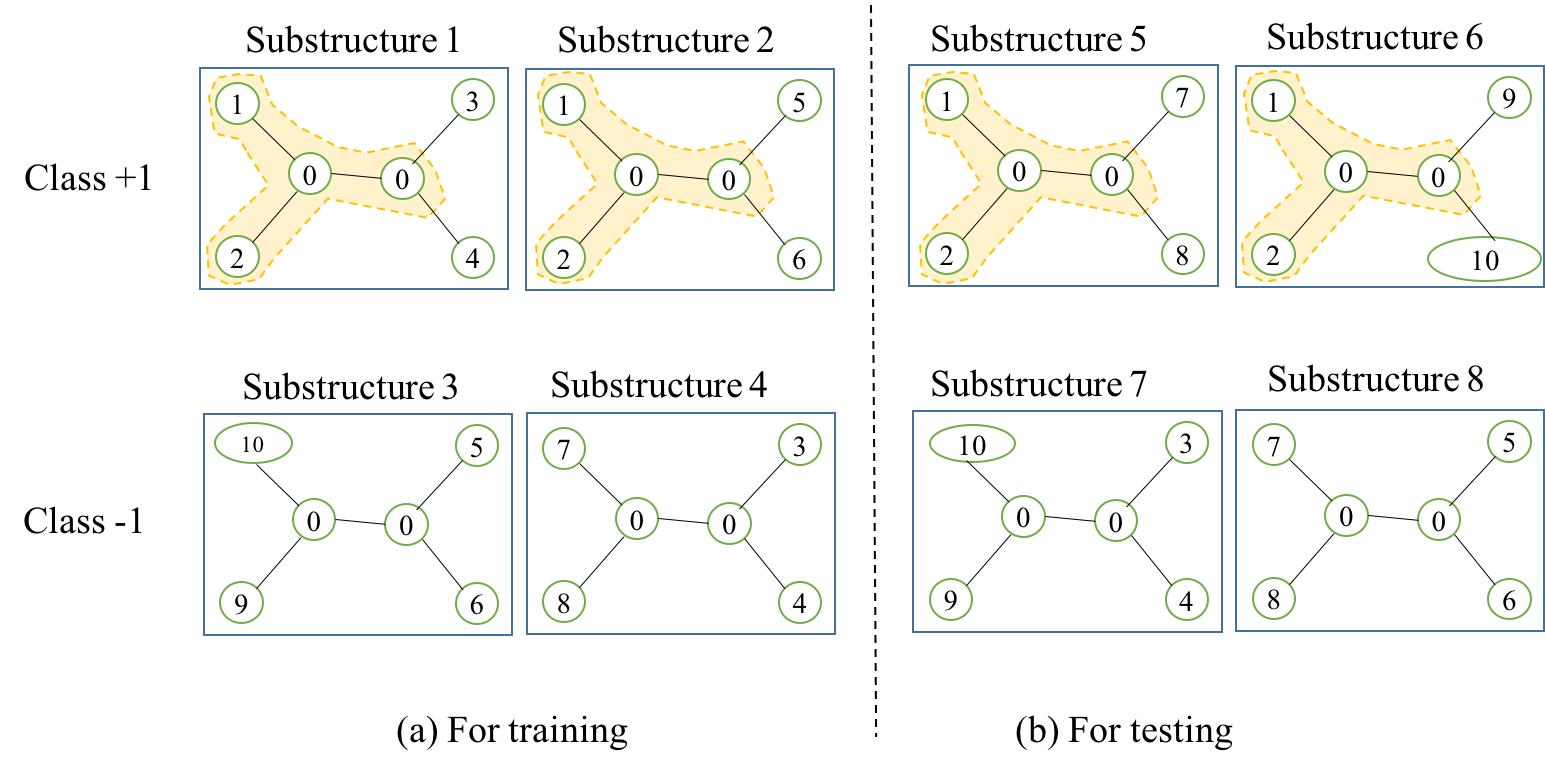}}
	\caption{Designed substructures 1-8: substructures 1,2,5 and 6 are indicative to positive class as they contain pattern '1-0(-2)-0' (emphasized in yellow). The rest are indicative to negative class. Graphs are generated from substructures by adding random nodes and noisy edges (20 examples for each substructure). Graphs generated from the substructures 1, 2, 3 and 4 are used for training. Graph generated from the substructures 5,6,7 and 8 are used for testing.} 
	\label{Fig:synthetic_subs}
\end{figure}

\begin{figure}[t]
    \centering
	\centerline{\includegraphics[width=1.0\columnwidth]{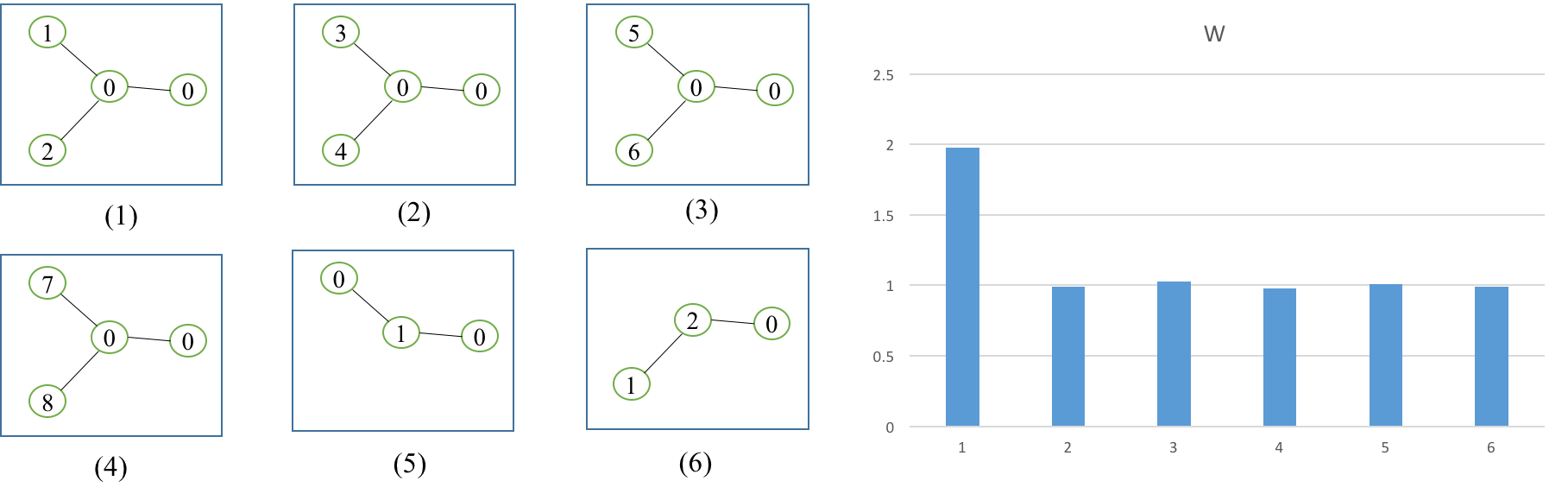}}
	\caption{Examples of selected subtree patterns of the first level $h=1$ (left) and their weights learned by the proposed learning algorithm (right).} 
	\label{Fig:synthetic_w}
\end{figure}

\subsection{Real-world data}
In this subsection we present an experimental evaluation of the proposed method on real-world data. 
We report experimental results on four benchmark bioinformatics data sets, involving node-labeled graphs, particularly, MUTAG, PTC-MR, PROTEIN and NCI1.
The MUTAG dataset consists of graph structures of 188 chemical compounds which are
either mutagenic aromatic or heteroromatic nitro compounds and nodes can have 7 discrete labels. The PTC-MR dataset consists of 344 chemical compounds which are
known to cause or not cause cancer in rats and mice, and nodes can have 19 discrete labels. The PROTEIN dataset consists of relations between  secondary structure elements represented by nodes and neighborhood in the amino-acid sequence or in 3D space by edges, and nodes can have 3 discrete labels. The NCI1 dataset is a balanced 
subset of chemical compounds screened for their ability to inhibit the growth of a panel
of human tumor cell lines, and nodes can have 37 discrete labels. Some statistics of these data sets are shown in Table 1.

We compared the proposed method to several state-of-the-art graph kernels. Due to the large number of graph kernels in the literature, we selected representatives of the major families of
graph kernels. In particular, for the family of walk based kernels, we compared the proposed method to the fast random walk kernel \cite{kashima2003marginalized} that essentially counts the common labeled walks. For the family of path based graph kernels, we compared to the shortest path kernel \cite{borgwardt2005shortest}. For the family of WL based graph kernels, we compared to WL subtree \cite{vishwanathan2010graph}, WL-OA \cite{kriege2016valid} and WWL kernels \cite{togninalli2019wasserstein}. The WL based kernels have been shown to be superior to previous approaches. 

\begin{table}[]
    \centering
    \caption{Statistics of datasets used in experiments}
    \begin{tabular}{l l l l l l}
    \hline
    Datasets & \#Graphs & \#Classes & Avg. \text{card(V)} & Avg. \text{card(E)} & \#labels\\
    \hline
    MUTAG & 188 & 2 (125 vs. 63) & 17.9 & 39.6 & 7\\
    PTC-MR & 344 & 2 (192 vs. 152) & 25.6 & 51.9 & 19 \\
    PROTEIN & 1113 & 2 (663 vs. 450) & 39.1 & 145.63 & 3\\
    NCI1 & 4110 & 2 (2053 vs. 2057 ) & N/A & N/A & N/A\\
    \hline
    \end{tabular}
    
    \label{tab:my_label}
\end{table}

\begin{figure}[t]
    \centering
	\centerline{\includegraphics[width=0.9\columnwidth]{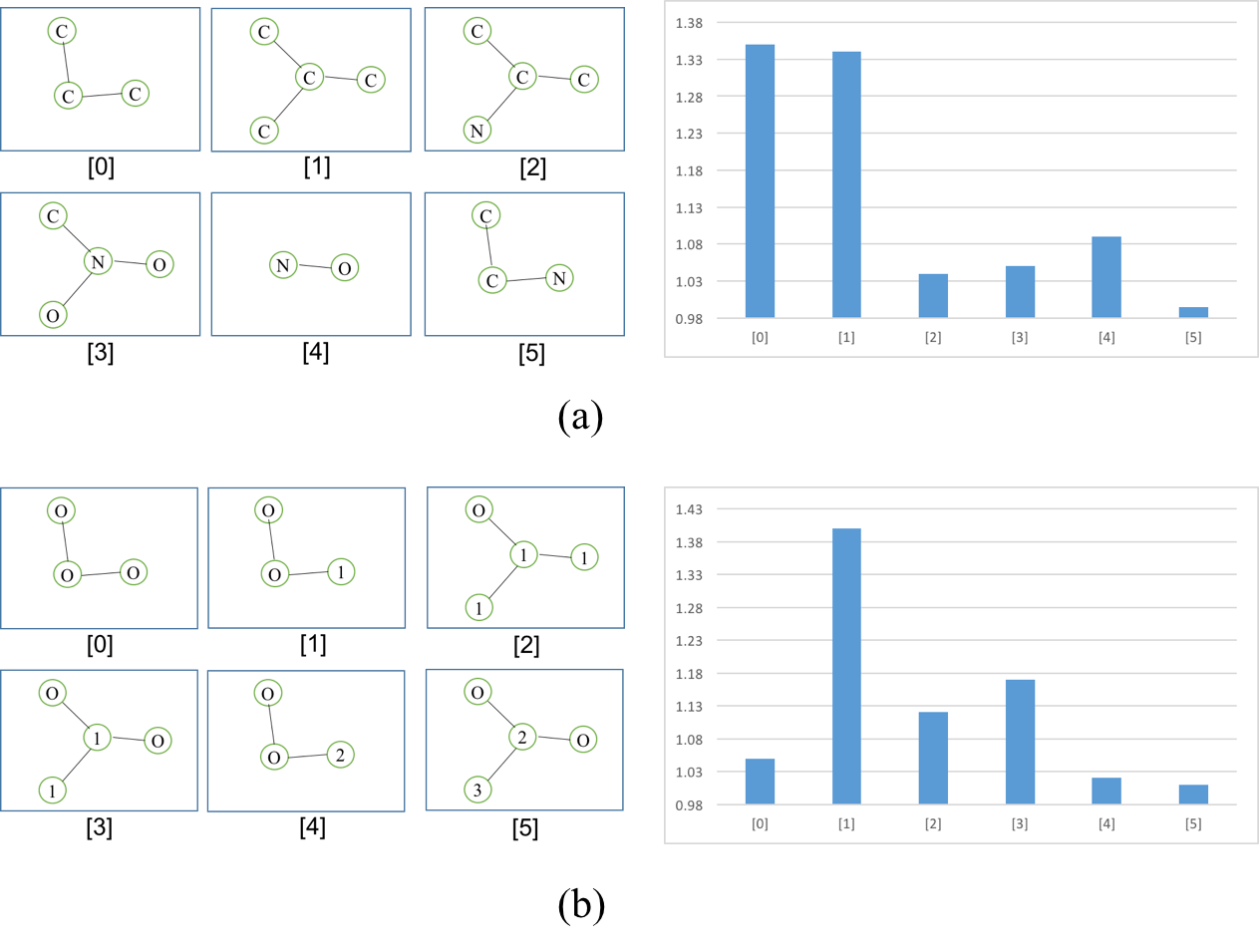}}
	\caption{Examples of selected subtree patterns and their weights of the first level $h=1$ (a) and the second level $h=2$ (b) extracted from MUTAG by Algorithm \ref{pSGD}.} 
	\label{Fig:mutag}
\end{figure}

\begin{figure}[t]
    \centering
	\centerline{\includegraphics[width=0.9\columnwidth]{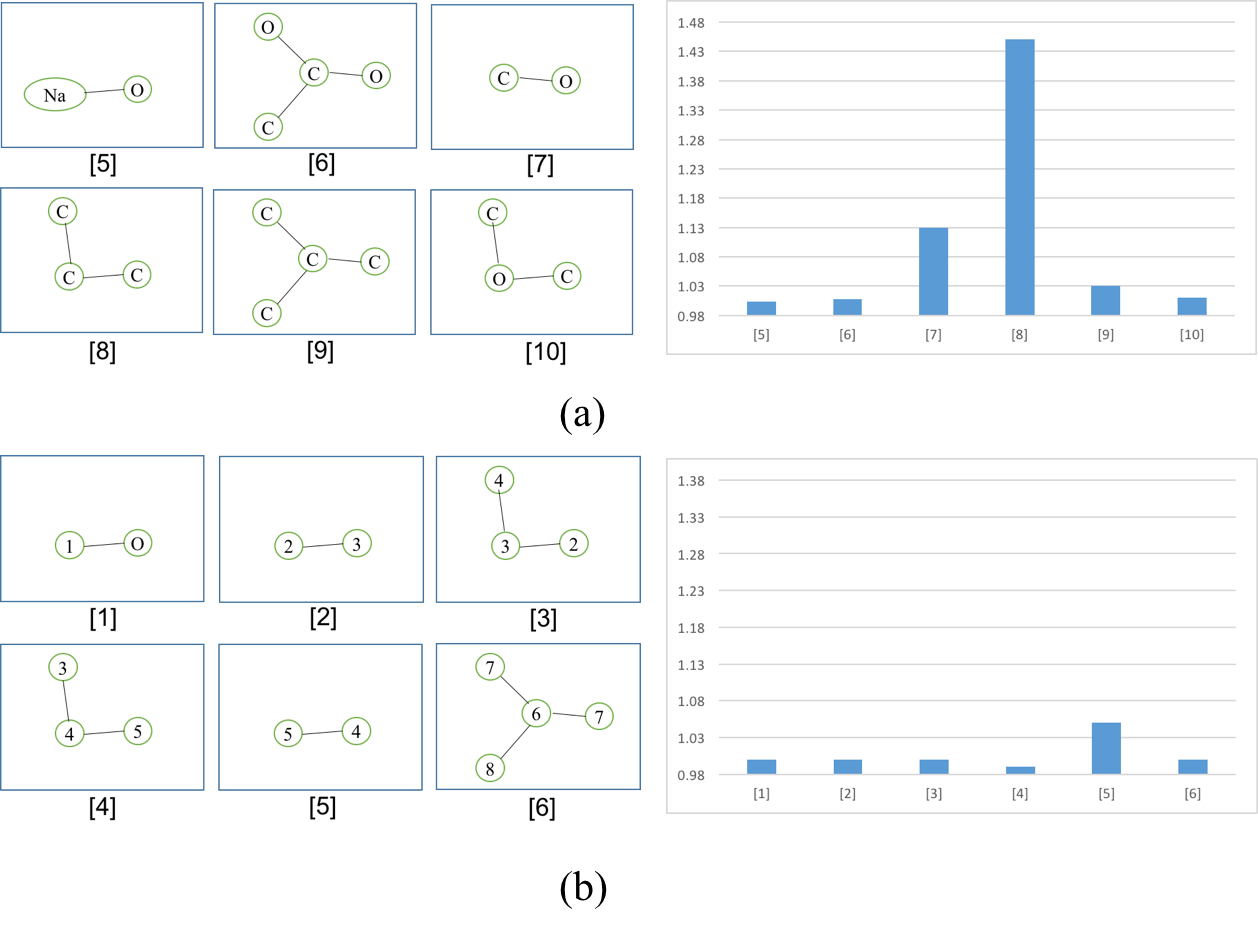}}
	\caption{Examples of selected subtree patterns and their weights of the first level $h=1$ (a) and the second level $h=2$ (b) extracted from PTC-MR by Algorithm \ref{pSGD}.} 
	\label{Fig:ptc}
\end{figure}

We report mean predictive accuracies and standard deviations obtained by 10-fold cross-validation repeated 10 times with random fold assignments. Within each fold, the number of hops $H\in\{1,2,...,6\}$ was selected by cross validation based on the training set.
The results evaluated by classification accuracy are summarised in Table 2. We used one-sided paired \textit{t-test} to verify if the accuracy differences between two methods on data sets are statistically significant.
We empirically observed that random walk and shortest path kernels were less competitive to WL-based kernels on four data sets. 
On three 
datasets MUTAG, PROTEIN and NCI1, the proposed method was comparable with WL-OA while it
improved its unweighted version WWL by 1.4\%, 1.5\% and 0.8\%, respectively (the calculated p-values were 0.061, 0.0087 and 0.055, respectively, smaller than the significance level of $\alpha=0.1$). On 
PTC-MR, the proposed method improved WWL by 0.6\% while outperforming WL-OA by nearly 
5\% (the calculated p-values were 0.0035 and $<$ 0.001, respectively). In all these data sets, random walk, shortest path and WL subtree kernels were dominated by the rest in large margins. Furthermore, we also investigated some selected subtree patterns at the first and second levels ($h=1,2$) along with their weights learned by the proposed algorithm from two data sets: MUTAG and PTC-MR (see Figures \ref{Fig:mutag} and \ref{Fig:ptc}, respectively). Interestingly, the weights were found different over substructures, indicating their different degrees of importance in the prediction task.
These experimental results confirmed the effectiveness of learning important
subtree patterns for WWL kernels.

\begin{table}[]
    \centering
    \caption{Classification accuracies and standard deviation on real-world graph data sets: MUTAG, PTC-MR, PROTEIN and NCI1.}
    \begin{tabular}{l c c c  c}
    \hline
    Kernels & MUTAG & PTC-MR & PROTEIN & NCI1\\
    \hline
    Random Walk \cite{kashima2003marginalized} & 85.06 $\pm$ 0.18 & 55.74 $\pm$ 3.64 & 71.11 $\pm$ 0.83 & 62.88 $\pm$ 0.22\\
    Shortest path \cite{borgwardt2005shortest} & 85.49 $\pm$ 0.59 & 53.29 $\pm$ 0.92 & 73.03 $\pm$ 1.13 & 61.36 $\pm$ 0.19\\
    WL Subtree \cite{shervashidze2009fast} & 85.61 $\pm$ 0.85 & 61.89 $\pm$ 1.97 & 72.5 $\pm$ 0.32 & 85.61 $\pm$ 0.13\\
    WL-OA \cite{kriege2016valid} & \textbf{88.17} $\pm$ 1.98 & 60.49 $\pm$ 1.39 & \textbf{75.89} $\pm$ 0.41 & \textbf{86.17} $\pm$ 0.35 \\
    WWL \cite{togninalli2019wasserstein} & 86.95 $\pm$ 1.35 & 64.86 $\pm$ 1.57 & 74.25 $\pm$ 0.74 & 85.69 $\pm$ 0.28\\
    Proposed & \textbf{88.37} $\pm$  1.82& \textbf{65.44} $\pm$ 0.97 & \textbf{75.73} $\pm$ 0.57 & \textbf{86.45} $\pm$ 0.11\\
    \hline
    \end{tabular}
    
    \label{tab:my_label}
\end{table}

\subsection{Computational efficiency of the proposed stochastic algorithm}
In this subsection, we evaluate the computational efficiency of the proposed Algorithm \ref{pSGD}. We empirically compared two variants: batch and  stochastic (Algorithm \ref{pSGD}), for solving the minimization problem (\ref{Eqn:ell}) in terms of running time. The first variant considers all pairs of graphs for every step of projected gradient descent. The second variant considers one pair of graphs at a time to take a single step.

First we assessed the running time of two variants on randomly generated graphs (as described in Subsection 5.1) with respect to two parameters: number of graphs $N$ and number of WL iterations $H$. We varied $N$ in range
$\{50,100,200,$
\noindent
$400,600,
800,1000\}$ and $H$ in range $\{1,2,3,4,5,6,7\}$. For each individual experiment, we fixed one parameter at its default value and varied the other. The default values were 100 for $N$ and 2 for $H$. We report CPU running times in seconds in Figure \ref{fig:runningtime_synthetic}. Empirically, we observed that the running time of full batch variant increased quickly when increasing the number of graphs $N$ and the number of WL iterations $H$. In contrast, the stochastic variant scaled much better with much lower running times, indicating that Algorithm \ref{pSGD} has high scalability for large scale data sets. The computational efficiency of Algorithm \ref{pSGD} can be explained by the fact that computing gradient of the loss function for a graph pair $g$ and $g^\prime$ involves a few substructures (or WL labels) $v$ shared by $g$ and $g^\prime$ in Eq. (\ref{Eqn:linearform}), i.e., sparsity of the feature vector $\textbf{Z}\left(g,g^\prime\right)$.

Second we assessed the running time of two variants on real-world data sets: MUTAG, PTC-MR, PROTEIN and NCI1. We reported the running time of two variants to finish the entire classification tasks, including learning subtree pattern weights, computing kernels and doing classification, in Table \ref{tab:runningtime_realworld}. The running time were taken average by 10-fold cross validation. We empirically observed that the full batch variant was slow when running on even small data sets MUTAG and PTC-MR, taking in approximately 4 hours and 9 hours, respectively. But the stochastic variant was much faster, taking only less than 4 minutes on the two data sets. We also observed that the stochastic variant could easily scale up to data sets with thousands of graphs. Particularly, on data sets PROTEIN and NCI1, the tasks were performed in nearly 1h 30' and 5h, respectively. However, it was impossible for the full batch variant to finish the tasks in less than 3 days. These evidence showed that the proposed stochastic variant is highly scalable.

\begin{figure}[t]

	\includegraphics[width=1.0\columnwidth]{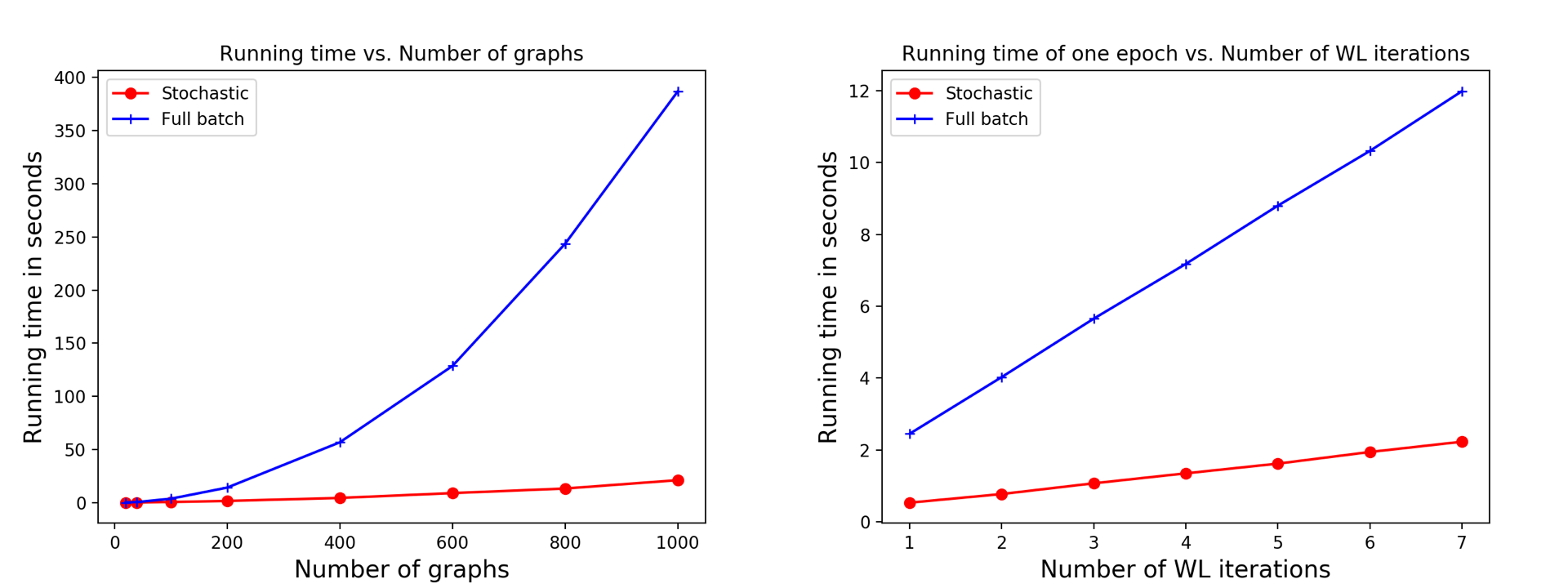}
	\caption{Running time in seconds on synthetic data sets of two variants: full batch and stochastic algorithms, for learning the subtree pattern importance $W$ (in the optimization problem (\ref{Eqn:ell})) (Default values: dataset size $N = 100$, WL iteration $H = 2$).} 
	\label{fig:runningtime_synthetic}
\end{figure}

\begin{table}[]
    \centering
    \caption{Running time in seconds of two variants: full batch and stochastic on real-world data sets: MUTAG, PTC-MR, PROTEIN and NCI1.}
    \begin{tabular}{l c c c  c}
    \hline
    variants/ data sets & MUTAG & PTC-MR & PROTEIN & NCI1\\
    \hline
    Full batch & 4h 29' & 8h 50' &  $\textgreater$ 3 days &  $\textgreater$ 1 week\\
    Stochastic & 1' 32'' & 3' 48'' & 1h 20'40''  & 5h 5'\\
    \hline
    \end{tabular}
    
    \label{tab:runningtime_realworld}
\end{table}

\section{Conclusion and Discussion}
In this work, we proposed to learn the weights of substructures of graphs, particularly, \emph{subtree patterns} (extracted by WL labeling scheme), to overcome the limitations of current graph kernels. We considered the problem of incorporating subtree pattern weights for WWL kernel \cite{togninalli2019wasserstein} by formulating the parametric form of Wasserstein distance taking into account subtree pattern weights, and learning these weights from data optimally for the tasks. 

Our proposed method has several advantages. First, it can learn the importance of subtree patterns specifically for the tasks through their weights in the parametric distance function. Second, the kernels converted from the learned parametric function of subtree pattern weights are valid. Third, the efficient stochastic algorithm for learning the weights has high scalability for large scale data sets, and its theoretical guarantees are provided. 

Although we considered WWL kernel for extracting important subtree patterns, an interesting and worthwhile extension of our work would be to apply this idea to other WL based graph kernels such as WL subtree and WL-OA kernels. The improvements of the optimization algorithm for learning subtree pattern weights in terms of both convergence and efficiency would also be our future work.

%
%

\bibliographystyle{spmpsci}      
\bibliography{mlj}
%
%
\appendix
\section{Appendices}
\subsection{Proof of Lemma \ref{boundonparams}}.
\begin{proof}
We prove the lemma by following the notion of using the same randomness for two data set $D_{n}$ and $D_{n,k}$ as in \cite{hardt2016train}.
Particularly, we supply the sample sequences $S=\{p_{1}=(z_{i_{1}},z_{j_{1}}),...,p_{T}=(z_{i_{T}},z_{j_{T}}) \}$ to two  identical learning algorithms except that for some $t$ ($1\leq t \leq T$), if
$p_{t}$ contains $z_{k}$ ($p_{t}=(z_{k}, z_{j_{t}})$ or $(z_{i_{t}}, z_{k})$), we replace it with $\hat{p}_{t}=(\hat{z}_{k}, z_{j_{t}})$ or $(z_{i_{t}}, \hat{z}_{k})$. So, there are two cases to consider:\\
\noindent
\textbf{Case 1:} At step $t$, Algorithm \ref{pSGD} picks a pair of samples ($z$, $z^\prime$) that contain no $z_{k}$ ($z\neq z_{k}$ and $z^\prime\neq z_{k}$) and this
case occurs with probability $(1-\frac{1}{n})^{2}$.
Then, we have:
\begin{align*}
    &\left\Vert \textbf{W}^{(t+1)}_{n}-\textbf{W}^{(t+1)}_{n,k}\right\Vert_{2}^{2}  = \left\Vert \mathrm{proj}_{\mathcal{C}}\left[\textbf{W}^{(t)}_{n} - \mu\nabla\ell\left(\textbf{W}^{(t)}_{n}, z, z^\prime\right)\right] -\mathrm{proj}_{\mathcal{C}}\left[\textbf{W}^{(t)}_{n,k} - \mu\nabla\ell\left(\textbf{W}^{(t)}_{n,k}, z, z^\prime\right)\right]\right\Vert_{2}^{2}\\
    &\leq \left\Vert \textbf{W}^{(t)}_{n} - \mu\nabla\ell\left(\textbf{W}^{(t)}_{n}, z, z^\prime\right) -\textbf{W}^{t}_{n,k} + \mu\nabla\ell\left(\textbf{W}^{(t)}_{n,k}, z, z^\prime\right)\right\Vert_{2}^{2}\\
    &= \left\Vert \textbf{W}^{(t)}_{n}-\textbf{W}^{(t)}_{n,k}\right\Vert_{2}^{2} - 2\mu\left(\textbf{W}^{(t)}_{n}-\textbf{W}^{(t)}_{n,k}\right)\left(\nabla\ell\left(\textbf{W}^{(t)}_{n}, z, z^\prime\right)-\nabla\ell\left(\textbf{W}^{(t)}_{n,k}, z, z^\prime\right)\right)\\
    &+ \mu^{2}\left\Vert \nabla\ell(\textbf{W}^{(t)}_{n}, z, z^\prime)-\nabla\ell\left(\textbf{W}^{(t)}_{n,k}, z, z^\prime\right)\right\Vert_{2}^{2}\\
    &\leq \left\Vert \textbf{W}^{(t)}_{n}-\textbf{W}^{(t)}_{n,k}\right\Vert_{2}^{2} - \mu\left(\frac{2}{\beta}-\mu\right)\left\Vert\nabla\ell\left(\textbf{W}^{(t)}_{n}, z, z^\prime\right)-\nabla\ell\left(\textbf{W}^{(t)}_{n,k}, z, z^\prime\right)\right\Vert_{2}^{2}
\end{align*}
The second line is obtained by the fact that $\left\Vert \mathrm{proj}_{\mathcal{C}}[\textbf{u}]-\mathrm{proj}_{\mathcal{C}}[\textbf{v}]\right\Vert_{2}\leq \left\Vert \textbf{u}-\textbf{v}\right\Vert_{2}$ for $\textbf{u},\textbf{v}$ in the domain. The last line is obtained by the $\beta$-smoothness of function $\ell$ (see Lemma \ref{propertiesofell}). So we have the following inequality (by selecting $\mu\leq \frac{2}{\beta}$):
\begin{equation}
\label{case1}
    \left\Vert \textbf{W}^{(t+1)}_{n}-\textbf{W}^{(t+1)}_{n,k}\right\Vert_{2} \leq \left\Vert \textbf{W}^{(t)}_{n}-\textbf{W}^{(t)}_{n,k}\right\Vert_{2}
\end{equation}
\textbf{Case 2:} At step $t$, Algorithm \ref{pSGD} picks a pair of samples ($z$, $z^\prime$) that contain $z_{k}$ ($z= z_{k}$ and $z^\prime= z_{k}$) and this case occurs with probability $1-(1-\frac{1}{n})^{2}$. Then we have:
\begin{equation}
\label{case2}
    \left\Vert \textbf{W}^{(t+1)}_{n}-\textbf{W}^{(t+1)}_{n,k}\right\Vert_{2} \leq  \left\Vert \textbf{W}^{(t)}_{n}-\textbf{W}^{(t)}_{n,k}\right\Vert_{2} + 2\mu L
\end{equation}
The above inequality holds as the norm of gradient of loss function $\ell$ is upper bounded by $L$.
From two inequalities (\ref{case1}) and (\ref{case2}), and considering the probabilities of two cases, we have the following inequality:
\begin{align}
    \textbf{E}_{\text{SGD}}\left[ \left\Vert \textbf{W}^{(t+1)}_{n}-\textbf{W}^{(t+1)}_{n,k}\right\Vert_{2} \right] \leq \left(1-\frac{1}{n}\right)^{2}\textbf{E}_{\text{SGD}}\left[\left\Vert \textbf{W}^{(t)}_{n}-\textbf{W}^{(t)}_{n,k}\right\Vert_{2}\right]\\
    +\left(1-\left(1-\frac{1}{n}\right)^{2}\right)\left(\textbf{E}_{\text{SGD}}\left[\left\Vert \textbf{W}^{(t)}_{n}-\textbf{W}^{(t)}_{n,k}\right\Vert_{2}\right]+2\mu L \right)\\
    = \textbf{E}_{\text{SGD}}\left[\left\Vert \textbf{W}^{(t)}_{n}-\textbf{W}^{(t)}_{n,k}\right\Vert_{2}\right] + 2\mu L \left( 1-\left(1-\frac{1}{n}\right)^{2}\right)
\end{align}
By the above recursive formula, we obtain the following:
\begin{align}
    \textbf{E}_{\text{SGD}}\left[\left\Vert \textbf{W}^{(T)}_{n}-\textbf{W}^{(T)}_{n,k}\right\Vert_{2}\right] \leq 2\left(1-\left(1-\frac{1}{n}\right)^{2}\right)\mu T L \leq \frac{4}{n}\mu T L
\end{align}
which completes the proof.
\end{proof}
\subsection{Proof of Lemma \ref{Lemma:BoundKn}}
\begin{proof}
By the definition of $\mathbb{K}_{n}$ as in (\ref{expectedgap}), we have the following:
\begin{align*}
    &\textbf{E}_{D_{n}}[\mathbb{K}_{n}] = \textbf{E}_{D_{n}}\textbf{E}_{\text{SGD}}[R(\textbf{W}_{n})-\textit{R}_{D_{n}}(\textbf{W}_{n})]\\
    &= \textbf{E}_{D_{n}}\textbf{E}_{\text{SGD}}\textbf{E}_{z,z^\prime}\ell(\textbf{W}_{n},z,z^\prime)-\textbf{E}_{D_{n}}\textbf{E}_{\text{SGD}}\frac{1}{n^{2}}\sum_{i,j=1}^{n}\ell(\textbf{W}_{n}, z_{i}, z_{j})\\
    &=\textbf{E}_{D_{n},z,z^\prime}\textbf{E}_{\text{SGD}}\left[ \frac{1}{n^{2}} \sum_{k=1}^{n}\sum_{j=1}^{n}\left[ \ell(\textbf{W}_{n},z,z^\prime)-\ell(\textbf{W}_{n},z_{k},z^\prime)+\ell(\textbf{W}_{n},z_{k},z^\prime)-\ell(\textbf{W}_{n},z_{k},z_{j}) \right] \right]\\
    &=\textbf{E}_{D_{n},z,z^\prime}\textbf{E}_{\text{SGD}}\left[ \frac{1}{n^{2}} \sum_{k=1}^{n}\sum_{j=1}^{n}\left[ \ell(\textbf{W}_{n},z,z^\prime)-\ell(\textbf{W}_{n},z_{k},z^\prime)+\ell(\textbf{W}_{n},z_{k},z^\prime)-\ell(\textbf{W}_{n},z_{k},z_{j}) \right] \right]\\
    &=\underbrace{\textbf{E}_{D_{n},z,z^\prime}\textbf{E}_{\text{SGD}}\left[ \frac{1}{n^{2}} \sum_{k=1}^{n}\sum_{j=1}^{n}\left[ \ell(\textbf{W}_{n},z,z^\prime)-\ell(\textbf{W}_{n},z_{k},z^\prime)\right]\right]}_{(a)}\\
    &+\underbrace{\textbf{E}_{D_{n},z,z^\prime}\textbf{E}_{\text{SGD}}\left[ \frac{1}{n^{2}} \sum_{k=1}^{n}\sum_{j=1}^{n}\left[\ell(\textbf{W}_{n},z_{k},z^\prime)-\ell(\textbf{W}_{n},z_{k},z_{j}) \right] \right]}_{(b)}
\end{align*}

We first process the part (a) which is equivalent to the following:
\begin{align*}
    \text{(a)} &= \frac{1}{n}\sum_{k=1}^{n}\textbf{E}_{D_{n},z,z^\prime}\textbf{E}_{\text{SGD}}\left[ \ell(\textbf{W}_{n},z,z^\prime) \right]-\frac{1}{n}\sum_{k=1}^{n}\textbf{E}_{D_{n},z,z^\prime}\textbf{E}_{\text{SGD}}\left[ \ell(\textbf{W}_{n},z_{k},z^\prime) \right]\\
    &=\frac{1}{n}\sum_{k=1}^{n}\textbf{E}_{D_{n},z,z^\prime}\textbf{E}_{\text{SGD}}\left[ \ell(\textbf{W}_{n},z,z^\prime) \right]-\frac{1}{n}\sum_{k=1}^{n}\textbf{E}_{D_{n},\hat{z}_{k},z^\prime}\textbf{E}_{\text{SGD}}\left[ \ell(\textbf{W}_{n},z_{k},z^\prime) \right]\\
    &=\frac{1}{n}\sum_{k=1}^{n}\textbf{E}_{D_{n},z,z^\prime}\textbf{E}_{\text{SGD}}\left[ \ell(\textbf{W}_{n},z,z^\prime) \right]-\frac{1}{n}\sum_{k=1}^{n}\textbf{E}_{D_{n},z,z^\prime}\textbf{E}_{\text{SGD}}\left[ \ell(\textbf{W}_{n,k},z,z^\prime) \right]\\
    &\leq \frac{1}{n}\sum_{k=1}^{n}\textbf{E}_{D_{n},z,z^\prime}\textbf{E}_{\text{SGD}}\left[\ell(\textbf{W}_{n},z,z^\prime)-\ell(\textbf{W}_{n,k},z,z^\prime) \right]\leq k_{n}
\end{align*}
Similarly, we also prove that $\text(b)\leq k_{n}$ which completes the proof.
\end{proof}

\subsection{Proof of Lemma \ref{Lemma:BoundKminusKn}}

\begin{proof}
By the definition of $\mathbb{K}_{n}$ as in (\ref{expectedgap}), we have:
\begin{align*}
    |\mathbb{K}_{n}-\mathbb{K}_{n,k}|=|\textbf{E}_{\text{SGD}}[R(\textbf{W}_{n})-\textit{R}_{D_{n}}(\textbf{W}_{n})]-\textbf{E}_{\text{SGD}}[R(\textbf{W}_{n,k})-\textit{R}_{D_{n,k}}(\textbf{W}_{n,k})]|\\
    \leq \underbrace{|\textbf{E}_{\text{SGD}}R(\textbf{W}_{n})-\textbf{E}_{\text{SGD}}R(\textbf{W}_{n,k})|}_{\text{(c)}} + \underbrace{|\textbf{E}_{\text{SGD}}R_{D_{n}}(\textbf{W}_{n})-\textbf{E}_{\text{SGD}}R_{D_{n,k}}(\textbf{W}_{n,k})|}_{\text{(d)}}
\end{align*}
We bound the two terms (c) and (d) as follows:

\begin{align*}
    \text{(c)}&=|\textbf{E}_{\text{SGD}}\textbf{E}_{z,z^\prime}[\ell(\textbf{W}_{n},z,z^\prime)-\ell(\textbf{W}_{n,k},z,z^\prime)]|\\
    &\leq\textbf{E}_{z,z^\prime}|\textbf{E}_{\text{SGD}}[\ell(\textbf{W}_{n},z,z^\prime)-\ell(\textbf{W}_{n,k},z,z^\prime)]|\leq k_{n}
\end{align*}

\begin{align*}
    \text{(d)}&=|\textbf{E}_{\text{SGD}}\frac{1}{n^{2}}\sum_{z_{i}\in D_{n}}\sum_{z_{j}\in D_{n}}\ell(\textbf{W}_{n},z_{i},z_{j})-\textbf{E}_{\text{SGD}}\frac{1}{n^{2}}\sum_{z_{i}\in D_{n,k}}\sum_{z_{j}\in D_{n,k}}\ell(\textbf{W}_{n,k},z_{i},z_{j})|\\
    &=|\frac{1}{n^{2}}\sum_{i\neq k, j\neq k}\textbf{E}_{\text{SGD}}[\ell(\textbf{W}_{n},z_{i},z_{j})-\ell(\textbf{W}_{n,k},z_{i},z_{j})]\\
    &+\frac{1}{n^{2}}\sum_{i\neq k}\textbf{E}_{\text{SGD}}[\ell(\textbf{W}_{n},z_{i},z_{k})-\ell(\textbf{W}_{n,k},z_{i},z_{k})]\\
    &+\frac{1}{n^{2}}\sum_{j\neq k}\textbf{E}_{\text{SGD}}[\ell(\textbf{W}_{n},z_{k},z_{j})-\ell(\textbf{W}_{n,k},z_{k},z_{j})]\\
    &\leq \frac{(n-1)^{2}}{n^{2}}k_{n} + \frac{2(n-1)}{n^{2}}M < k_{n}+\frac{2M}{n}
\end{align*}
The two above inequalities complete the proof.
\end{proof}

\end{document}